\documentclass[nohyperref]{article}

\usepackage{microtype}
\usepackage{graphicx}
\usepackage{subfigure}
\usepackage{booktabs} 

\usepackage{hyperref}

\usepackage{url}
\usepackage{enumitem}
\usepackage{amsthm}
\usepackage{soul}
\usepackage{algorithm,algorithmic,multicol}

\newtheorem{prop}{Proposition}
\usepackage{xcolor}

\usepackage[accepted]{icml2022}

\usepackage{amsmath}
\usepackage{amssymb}
\usepackage{mathtools}
\usepackage{amsthm}

\usepackage[capitalize,noabbrev]{cleveref}

\theoremstyle{plain}

\theoremstyle{definition}

\theoremstyle{remark}

\usepackage[textsize=tiny]{todonotes}

\icmltitlerunning{Submission and Formatting Instructions for ICML 2022}

\begin{document}

\twocolumn[
\icmltitle{Communicating via Markov Decision Processes}



\icmlsetsymbol{equal}{*}

\begin{icmlauthorlist}
\icmlauthor{Samuel Sokota}{equal,cmu}
\icmlauthor{Christian Schroeder de Witt}{equal,oxf}
\icmlauthor{Maximilian Igl}{waymo}
\icmlauthor{Luisa Zintgraf}{oxf}
\icmlauthor{Philip Torr}{oxf}
\icmlauthor{Martin Strohmeier}{arma}
\icmlauthor{J. Zico Kolter}{cmu,bosch}
\icmlauthor{Shimon Whiteson}{oxf,waymo}
\icmlauthor{Jakob Foerster}{oxf}
\end{icmlauthorlist}

\icmlaffiliation{cmu}{Carnegie Mellon University}
\icmlaffiliation{oxf}{Oxford University}
\icmlaffiliation{waymo}{Waymo Research}
\icmlaffiliation{arma}{armasuisse Science + Technology}
\icmlaffiliation{bosch}{Bosch Center for AI}

\icmlcorrespondingauthor{Samuel Sokota}{ssokota@andrew.cmu.edu}
\icmlcorrespondingauthor{Christian Schroeder de Witt}{schroederdewitt@gmail.com}

\icmlkeywords{Machine Learning, ICML}

\vskip 0.3in
]



\printAffiliationsAndNotice{\icmlEqualContribution} 

\begin{abstract}
We consider the problem of communicating exogenous information by means of Markov decision process trajectories. This setting, which we call a Markov coding game (MCG), generalizes both source coding and a large class of referential games. 
MCGs also isolate a problem that is important in decentralized control settings in which cheap-talk is not available---namely, they require balancing communication with the associated cost of communicating.
We contribute a theoretically grounded approach to MCGs based on maximum entropy reinforcement learning and minimum entropy coupling that we call MEME.
Due to recent breakthroughs in approximation algorithms for minimum entropy coupling, MEME is not merely a theoretical algorithm, but can be applied to practical settings.
Empirically, we show both that MEME is able to outperform a strong baseline on small MCGs and that MEME is able to achieve strong performance on extremely large MCGs. 
To the latter point, we demonstrate that MEME is able to losslessly communicate binary images via trajectories of Cartpole and Pong, while simultaneously achieving the maximal or near maximal expected returns, and that it is even capable of performing well in the presence of actuator noise.
\end{abstract}

\section{Introduction}

This work introduces a novel problem setting called Markov coding games (MCGs).
MCGs are two-player decentralized Markov decision processes \citep{oliehoek_concise_2016} that proceed in four steps.
In the first step, one agent (called the sender) receives a special private observation (called the message), which it is tasked with communicating.
In the second step, the sender plays out an episode of a Markov decision process (MDP).
In the third, the other agent (called the receiver) receives the sender's MDP trajectory as its observation.
In the fourth, the receiver estimates the message.
The shared payoff to the sender and receiver is a weighted sum of the cumulative reward yielded by the MDP and an indicator specifying whether or not the receiver correctly decoded the message.

Among the reasons that MCGs are of interest is the fact that they generalize other important settings.
The first of these is referential games.
In a referential game, a sender attempts to communicate a message to a receiver using cheap talk actions---i.e., communicatory actions that do not have externalities on the transition or reward functions.
Referential games have been a subject of academic interest dating back at least as far as Lewis's seminal work \textit{Convention} \citep{lewis_convention_1969}. 
Since then, various flavors of referential games have been studied in game theory \citep{referential-gt}, artificial life \citep{referential-alife}, evolutionary linguistics \citep{referential-evoling}, cognitive science \citep{referential-cogsci}, and machine learning \citep{lazaridou_emergence_2018}.
MCGs can be viewed as a generalization of referential games to a setting in which the sender's actions may incur costs.

A second problem setting generalized by MCGs is source coding \citep{information-theory} (also known as data compression).
In source coding, the objective is to construct an injective mapping from a space of messages to the set of sequences of symbols (for some finite set of symbols) such that the expected output length is minimized.
Source coding has a myriad of real world applications involving the compression of images, video, audio, and genetic data.
MCGs can be viewed as a generalization of the source coding problem to a setting where the cost of an encoding may involve complex considerations, rather than simply being equal to the sequence length.

Yet another reason to be interested in MCGs is that they isolate an important subproblem of decentralized control.
In particular, achieving good performance in an MCG requires the sender's actions to simultaneously perform control in an MDP and communicate information (i.e., to communicate implicitly).
This presents a challenge for approximate dynamic programming, the foundation for preeminent approaches to constructing control policies \citep{sutton_reinforcement_2018}, because the values of communication protocols depend on counterfactuals,
violating dynamic programming's locality assumption.

To address MCGs, we propose a theoretically grounded algorithm called MEME. 
MEME leverages a union of maximum entropy reinforcement learning (MaxEnt RL) \citep{ziebart_maximum_2008} and minimum entropy coupling (MEC) \citep{mec-hardness}. 
The key insight is that maximizing the returns of the MDP can be disentangled from learning a good communication protocol by means of observing that the entropy of a policy corresponds (in an informal sense) to its capacity to communicate. 
MEME leverages this insight in two steps. 
In the first step, MEME constructs a MaxEnt policy for the MDP, balancing between maximizing expected return and maximizing cumulative conditional entropy.
This MaxEnt policy is provided to both the sender and the receiver.
In the second step, at each decision point, MEME uses recent breakthroughs in MEC approximation techniques \cite{mec-appx} to pair messages with actions in such a way that the sender selects actions with the same probabilities as the MaxEnt RL policy (thereby guaranteeing the same expected return from the MDP) and the receiver’s uncertainty about the message is greedily reduced as much as possible.
After the sender completes its MDP episode, the receiver uses the MaxEnt RL policy and the sender's trajectory to compute the posterior over the message and guesses the maximum a posteriori.

To demonstrate the efficacy of MEME, we present experiments for MCGs based on a gridworld,
Cartpole, and Pong \citep{ale}, which we call CodeGrid, CodeCart, and CodePong, respectively. 
For CodeGrid, we show that with a message space cardinality in the 10s or 100s, MEME significantly outperforms a baseline combining a reinforcement learning with an optimal receiver.
For CodeCart and CodePong, we consider a message space of binary images and a uniform distribution over messages, meaning that a randomly guessing receiver has an astronomically small probability of guessing correctly.
Remarkably, we show that MEME is able to achieve an optimal expected return in Cartpole and Pong while simultaneously losslessly communicating images to the receiver, demonstrating that MEME has the capacity to be scaled to extremely large message spaces and complex control tasks.
Moreover, we find that the performance of MEME decays gracefully as the amount of actuator noise in the environment increases.

\section{Related Work}

The works that are most closely related to this one can be taxonomized as coming from literature on multi-agent reinforcement learning, coding, and diverse skill learning.

\paragraph{Multi-Agent Reinforcement Learning} One body of related work investigates multi-agent reinforcement learning settings in which the agents are endowed with explicit communication channels.
\citet{learning-to-comm} propose DIAL, which optimizes the sender's protocol by performing gradient ascent through the parameters of the receiver; \citet{ma-backprop} propose CommNet, which uses mean-pooling to process messages to handle dynamic number of agents over multiple steps;
\citet{tarmac} propose TarMAC, which also handles targeted messaging;
\citet{mao,imac,mao2} investigate handling settings in which the bandwidth of the communication channel is limited.
\citet{varun} propose an approach for discrete communicate channels in which the sender's behavioral policy deterministically selects the action that maximizes the receiver's posterior probability of the correct message, when computed using the target policy.
The problem setting we examine differs from this line of work in that the communication channel is not costless---instead, it possesses costs dictated by an MDP.
More similar to our work in this respect is that of \citet{strouse}, who investigate auxiliary rewards for taking actions with high mutual information.
The baseline for our CodeGrid experiments loosely resembles \citeauthor{strouse}'s algorithm.

\paragraph{Coding} Another body of related work concerns extensions of the source coding problem.
Length limited coding \citep{length-limited,kostina_2011} considers a problem setting in which the objective is to minimize the expected sequence length (as before), subject to a maximum length constraint.
Coding with unequal symbol costs \citep{unequal-letter-costs,channel-coding-with-costs} considers the problem in which the goal is to minimize the expected cumulative symbol cost of the sequence to which the message is mapped.
The cost of a symbol may differ from the cost of other symbols arbitrarily, making it a strictly more general problem setting than standard source coding (which can alternatively be thought of as minimizing cumulative symbol cost with equally costly symbols).
Both length limited coding and coding with unequal costs are subsumed by Markov coding games.
And while existing algorithms for both standard source coding and the extensions above are well-established and widely commercialized, they do not address the more general MCG setting.

MCGs are also related to finite state Markov channel settings \citep{fsmc} and Markov channel settings with feedback \citep{markov_feedback,markov_feedback2}.
In such settings, the fidelity of the channel by which the sender communicates to a receiver is controlled by a Markov process.
Another related setting is intersymbol interference, where the sender's previously selected symbols (i.e., actions) may cause interference with subsequently selected symbols, making them less likely to be faithfully transmitted to the receiver \citep{isi}.
MCGs differ from both Markov channel and intersymbol interference settings in that the Markov system controls the cost of the channel, rather than its transmission quality.
MCGs are more resemblant of a setting in which the channel is reliable, but subject to natural variation in costs, such as based on weather or third party usage, as well as variation based on the sender's own usage.

\paragraph{Diverse Skill Learning} A third area of related research is that of diverse skill learning \citep{diayn}.
\citet{diayn} propose an unsupervised learning method for discovering diverse, identifiable skills.
Their objective, called DIAYN, seeks to learn diverse, discriminable skills.
This paradigm resembles our work in the sense that skills can be interpreted as messages and discriminability can be interpreted as maximizing the mutual information between the skill and the state.
The baseline used in our CodeGrid experiments can also be viewed as an adaptation of an idealized version of DIAYN to the MCG setting.

\section{Background and Notation}

We will require the following background material.

\paragraph{Markov Decision Processes} To represent our task formalism, we use finite Markov decision processes (MDPs). 
We notate MDPs using tuples $\langle \mathcal{S}, \mathcal{A}, \mathcal{R}, \mathcal{T} \rangle$ where $\mathcal{S}$ is the set of states, $\mathcal{A}$ is the set of actions, $\mathcal{R} \colon \mathcal{S} \times \mathcal{A} \to \mathbb{R}$ is the reward function, and $\mathcal{T} \colon \mathcal{S} \times \mathcal{A} \to \Delta(\mathcal{S})$ is the transition function.
An agent's interactions with an MDP are dictated by a policy $\pi \colon \mathcal{S} \to \Delta(\mathcal{A})$ mapping states to distributions over actions.
We focus on episodic MDPs, meaning that after a finite number of transitions have occurred, the MDP will terminate.
The history of states and actions is notated using $h = (s^0, a^0, \dots, s^t)$, or $H = (S^0, A^0, \dots, S^t)$, if it is random.
We use the notation $\mathcal{R}(h) = \sum_j \mathcal{R}(s^j, a^j)$ to denote the amount of reward accumulated over the course of a history.
When a history is terminal, we use $z$ (or $Z$, if it is random) to notate it, rather than $h$.
The objective of an MDP is to construct a policy yielding a large cumulative reward in expectation $\mathbb{E}[\mathcal{R}(Z) \mid \pi]$.

\begin{figure}
    \centering
    \includegraphics[width=0.9\linewidth]{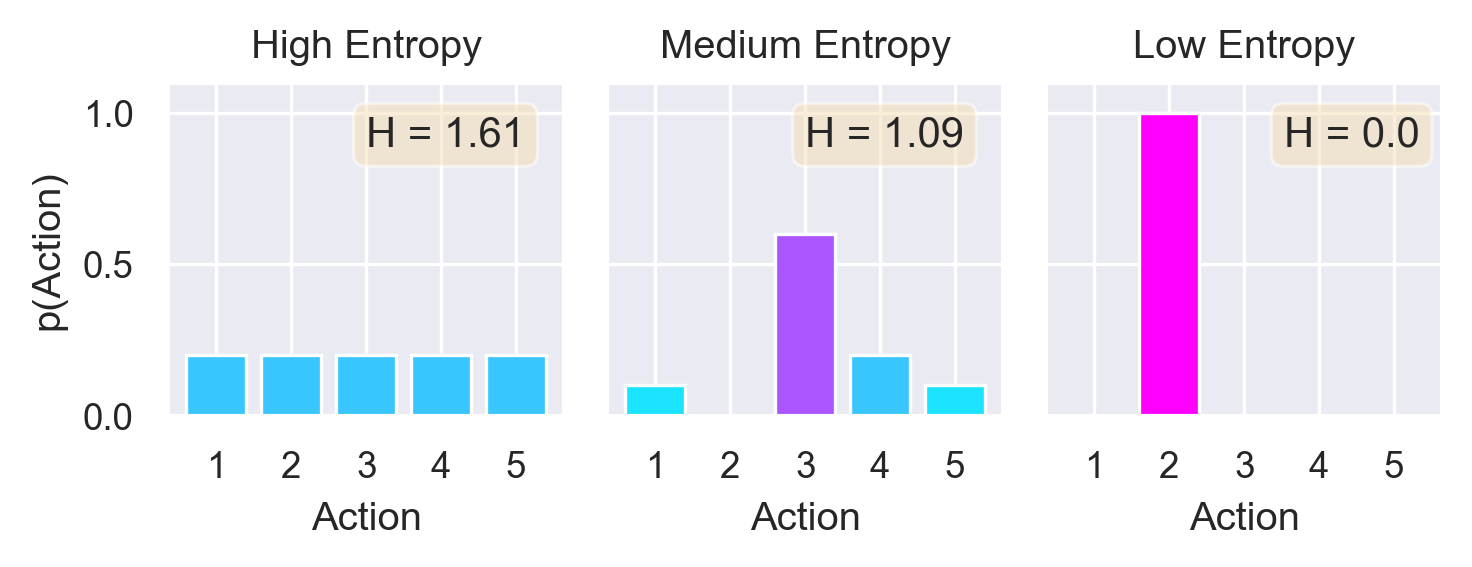}
    \caption{Three probability distributions over actions, in order of decreasing entropy (left to right).}
    \label{fig:entropy}
\end{figure}

\paragraph{Entropy} To help us quantify the idea of uncertainty, we use entropy.
The entropy of a random variable $X$ is $\mathcal{H}(X) = - \mathbb{E}\, \log \mathcal{P}(X)$.
It is maximized when the mass of $\mathcal{P}_{X}$ is spread as evenly as possible and minimized when the mass of $\mathcal{P}_{X}$ is concentrated at a point.

In the context of decision-making, entropy can be used to describe the uncertainty regarding which action will be taken by an agent (see Figure \ref{fig:entropy}).
When a policy spans multiple decision points, the uncertainty regarding the agent's actions given that the state is known is naturally described by conditional entropy.
Conditional entropy is the entropy of a random variable, conditioned upon the fact that the realization of another random variable is known.
More formally, conditional entropy is defined by $\mathcal{H}(X \mid Y) = \mathcal{H}(X, Y) - \mathcal{H}(Y)$
where the joint entropy $\mathcal{H}(X, Y) = - \mathbb{E} \, \log \mathcal{P}(X, Y)$ is defined as the entropy of $(X, Y)$ considered as a random vector.

In some contexts, it is desirable for a decision-maker's policy to be highly stochastic.
In such cases, an attractive alternative to the expected cumulative reward objective is the maximum entropy RL objective \cite{ziebart_maximum_2008}
$\max_{\pi} \mathbb{E}_{\pi} \left[\sum_t \mathcal{R}(S^t, A^t) + \alpha \mathcal{H}(A^t \mid S^t) \right]$,
which trades off between maximizing expected return and pursuing trajectories along which its actions have large cumulative conditional entropy, using the temperature hyperparameter $\alpha$.

\begin{figure}
    \centering
    \includegraphics[width=0.9\linewidth]{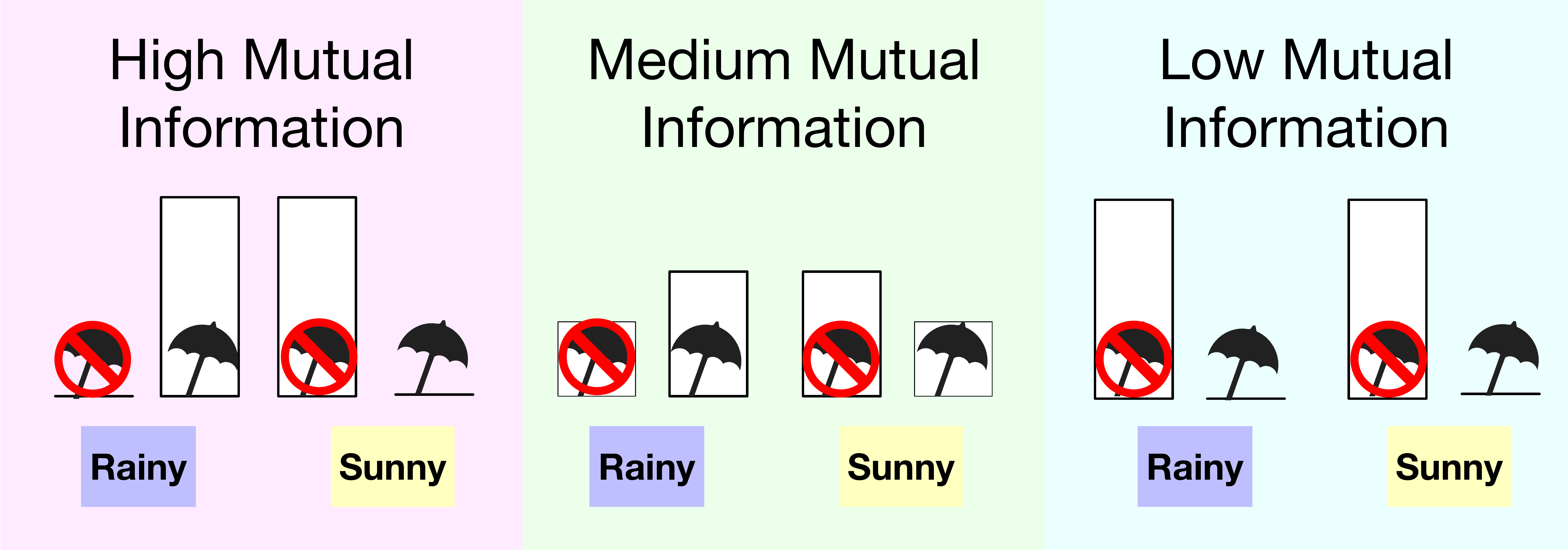}
    \caption{The mutual information between whether it rains and whether Alice (pink), Bob (green), and Charlie (teal) bring umbrellas, respectively.
    Alice always checks an unrealistically reliable weather station before leaving, so there is perfect mutual information.
    Bob makes a guess as to whether it will rain based on whether it is cloudy, so there is some mutual information.
    Charlie does not mind being rained on, so there is no mutual information.}
    \label{fig:mi}
\end{figure}

\paragraph{Mutual Information} A closely related concept to entropy is mutual information.
Mutual information describes the strength of the dependence between two random variables.
The greater the mutual information between two random variables, the more the outcome of one affects the conditional distribution of the other (see Figure \ref{fig:mi}).
Symbolically, mutual information is defined by 
$\mathcal{I}(X; Y) = \mathcal{H}(Y) {-} \mathcal{H}(Y \mid X) = \mathcal{H}(X) {-} \mathcal{H}(X \mid Y)$.
From this definition, we see explicitly that the mutual information of two random variables can be interpreted as the amount of uncertainty about one that is eliminated by observing the realization of the other.

Mutual information is important for communication because we may only be able to share the realization of an auxiliary random variable, rather than that of the random variable of interest.
In such cases, maximizing the amount of communicated information amounts to maximizing the mutual information between the auxiliary random variable and the random variable of interest.

\paragraph{The Data Processing Inequality} The independence relationships among random variables play an important role in determining their mutual information.
If random variables $X$ and $Z$ are conditionally independent given $Y$ (that is, $X \perp Z \mid Y$), the data processing inequality states that
$\mathcal{I}(X; Y) \geq \mathcal{I}(X; Z)$.
Less formally, the inequality states that if $Z$ does not provide additional information about $X$ given $Y$, then the dependence between $X$ and $Z$ cannot be stronger than the dependence between $X$ and $Y$.
The conditional independence required for the data processing inequality is depicted visually in Figure \ref{fig:dpe}.

\begin{figure}
    \centering
    \includegraphics[width=.7\linewidth]{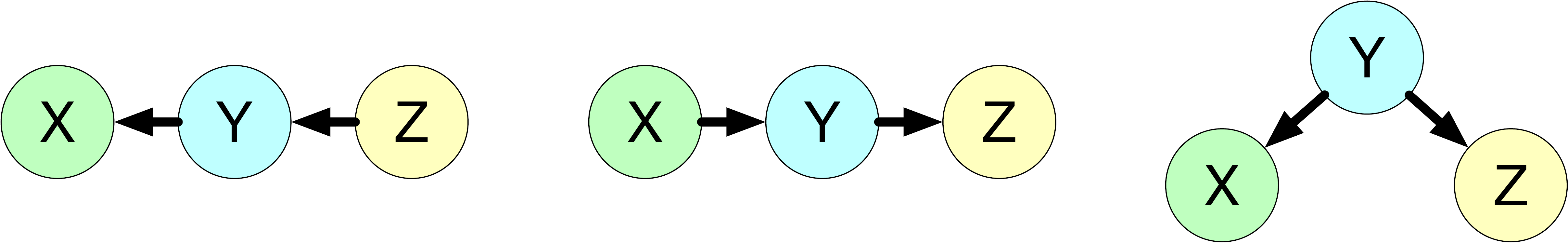}
    \caption{If any (or equivalently, all) of the graphs above are valid structures for a probabilistic graphical model for the joint distribution over $X, Y, Z,$ and there are no confounding variables, the data processing inequality states that the dependence between $X$ and $Y$ must be at least as strong as that of $X$ and $Z$.}
    \label{fig:dpe}
\end{figure}

\begin{figure}
    \centering
    \includegraphics[width=\linewidth]{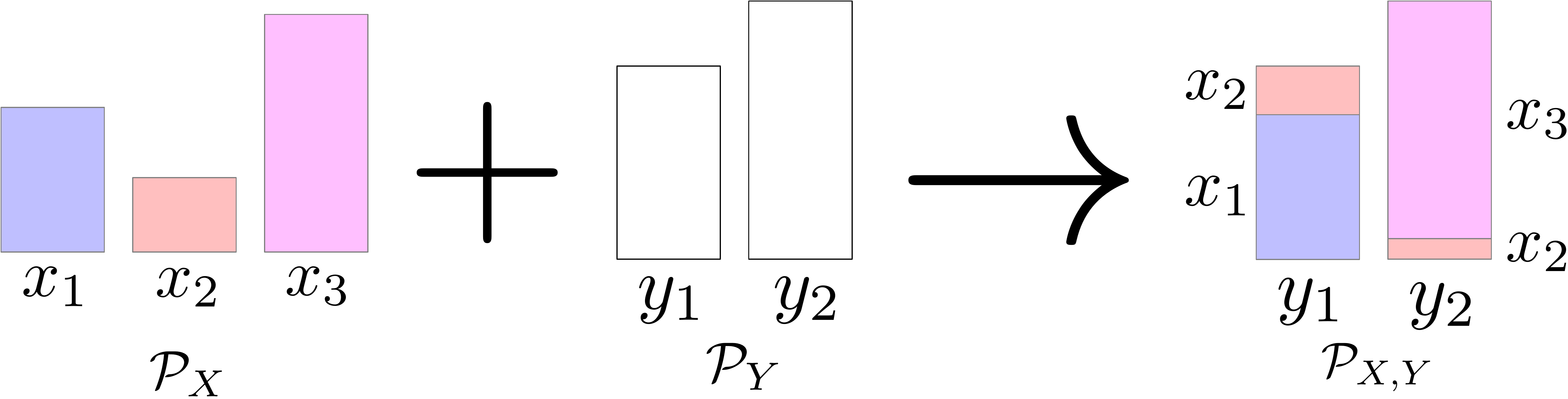}
    \caption{An approximate minimum entropy coupling.
    Given marginal distributions $\mathcal{P}_X$ and $\mathcal{P}_Y$, minimum entropy coupling constructs a joint distribution $\mathcal{P}_{X, Y}$ having minimal joint entropy. 
    }
    \label{fig:mec}
\end{figure}

\paragraph{Minimum Entropy Coupling} In some cases, we may wish to maximize the mutual information between two random variables subject to fixed marginals.
That is, we are tasked with determining a joint distribution $\mathcal{P}_{X, Y}$ that maximizes the mutual information $\mathcal{I}(X; Y)$ between $X$ and $Y$ subject to the constraints that $\mathcal{P}_{X, Y}$ marginalizes to $\mathcal{P}_X$ and $\mathcal{P}_Y$, where $\mathcal{P}_X$ and $\mathcal{P}_Y$ are given as input.
Invoking the relationship between mutual information and joint entropy $\mathcal{I}(X; Y) = \mathcal{H}(X) + \mathcal{H}(Y) - \mathcal{H}(X, Y)$, we see that this problem is equivalent to that of minimizing the joint entropy of $X$ and $Y$.
As a result, this problem is referred to as the \textit{minimum entropy coupling problem}.
A visual example is shown in Figure \ref{fig:mec}.
While minimum entropy coupling is NP-hard \cite{mec-hardness}, \citet{mec-appx} recently showed that there exists a polynomial time algorithm that is suboptimal by no more than one bit.

\section{Markov Coding Games}

We are now ready to introduce Markov coding games (MCGs). An MCG is a tuple $\langle (\mathcal{S}, \mathcal{A}, \mathcal{T}, \mathcal{R}), \mathcal{M}, \mu, \zeta \rangle$, where $(\mathcal{S}, \mathcal{A}, \mathcal{T}, \mathcal{R})$ is a Markov decision process, $\mathcal{M}$ is a set of messages, $\mu$ is a distribution over $\mathcal{M}$ (i.e., the prior over messages), and $\zeta$ is a non-negative real number we call the message priority.
\textbf{An MCG proceeds in the following steps:}
\begin{enumerate}[leftmargin=*, nosep]
    \item First, a message $M \sim \mu$ is sampled from the prior over messages and revealed to the sender.
    \item Second, the sender uses a message conditional policy $\pi_{\mid M}$, which takes states $s \in \mathcal{S}$ and messages $m \in \mathcal{M}$ as input and outputs distributions over MDP actions $\Delta(\mathcal{A})$, to generate a trajectory $Z \sim (\mathcal{T}, \pi_{\mid M})$ from the MDP.
    \item Third, the sender's terminal MDP trajectory $Z$ is given to the receiver as an observation.
    \item Fourth, the receiver uses a terminal MDP trajectory conditional policy $\pi_{\mid Z}$, which takes terminal trajectories $z \in \mathcal{Z}$ as input and outputs distributions over messages $\Delta(\mathcal{M})$, to estimate the message $\hat{M} \sim \pi_{\mid Z}(Z)$.
\end{enumerate}
The objective of the agents is to maximize the expected weighted sum of the return and the accuracy of the receiver's estimate $\mathbb{E} \left[\mathcal{R}(Z) + \zeta \mathbb{I}[M = \hat{M}] \mid \pi_{\mid M}, \pi_{\mid Z} \right]$.
Optionally, in cases in which a reasonable distance function is available, we allow for the objective to be modified to minimizing the distance between the message and the guess $d(M, \hat{M})$, rather than maximizing the probability that the guess is correct. A visual depiction of the structure MCGs is given in Figure~\ref{fig:implicit-referential-game}.
\begin{figure}
    \centering
    \includegraphics[width=\linewidth]{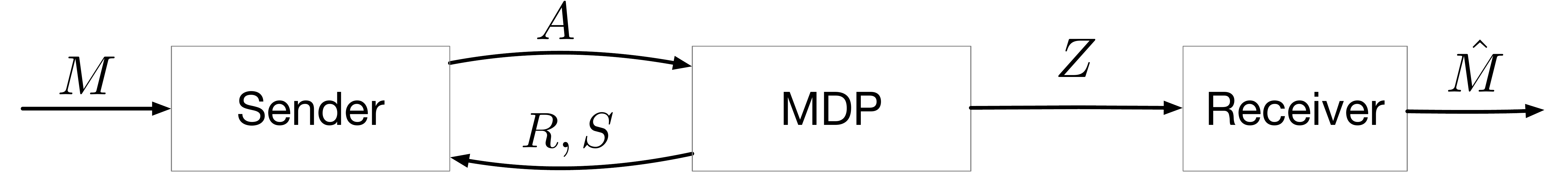}
    \caption{
    \textbf{A depiction of the structure of MCGs.}
    First, the sender is given a message.
    Second, the sender is tasked with a MDP, unrelated to the message.
    Third, the receiver observes the sender's trajectory.
    Fourth, the receiver estimates the message.
    }
    \label{fig:implicit-referential-game}
\end{figure}

\subsection{An Example}
\begin{table}
\centering
\caption{A payoff matrix for a simple MCG.}
\begin{tabular}{|c|c|cc|}
\hline
 &  & Receiver &\\
Message & Sender & $\hat{m}$ & $\hat{m}'$\\
\hline
 & $a$ & (4, $\zeta$) & (4, 0)\\
 $m$ & $a'$ & (3, $\zeta$) & (3, 0)\\
 & $a''$ & (0, $\zeta$) & (0, 0)\\
\hline
 & $a$ & (4, 0) & (4, $\zeta$)\\
 $m'$ & $a'$ & (3, 0) & (3, $\zeta$)\\
 & $a''$ & (0, 0) & (0, $\zeta$)\\
 \hline
\end{tabular}
\label{fig:mcgex}
\end{table}

A payoff matrix for a simple MCG is shown in Table \ref{fig:mcgex}.
In this game, the sender is given one of two messages $m$ and $m'$, with equal probability.
It then chooses among three actions $a$, $a'$ and $a''$, for which the rewards are $4$, $3$, and $0$, respectively.
The receiver observes the sender's trajectory (which is equivalent to the sender's action in MDPs with one state) and estimates the message using actions $\hat{m}$ and $\hat{m}'$, corresponding to guessing to $m$ and $m'$, respectively.
The receiver accrues a reward of $\zeta$ for guessing correctly.
The payoff entries in table denote $(\mathcal{R}(Z), \zeta \mathbb{I}[M = \hat{M}])$ for each outcome.

As is generally true of MCGs, this MCG is difficult for independent approximate dynamic programming-based approaches because their learning dynamics are subject to local optima.
Consider a run in which the sender first learns to maximize its immediate reward by always selecting $a$.
Now, the receiver has no reason to condition its guess on the sender's action because the sender is not communicating any information about the message.
As a result, thereafter, the sender has no incentive to communicate any information in its message, because the receiver would ignore it anyways.
This outcome, sometimes called a babbling equilibrium \citep{babbling}, leads to a total expected return of $4 + \zeta/2$ (sender always selects $a$, receiver guesses randomly).
In cases in which the message priority $\zeta$ is small (i.e., communication is unimportant), the babbling equilibrium performs well.
However, it can perform arbitrarily poorly as $\zeta$ becomes large.

\subsection{Special Cases}

We can express both referential games and various source coding settings as special cases of the MCG formalism by describing the MDPs to which they correspond.

\textbf{(A Large Class of) Referential Games} We can express a $T$ step referential game as follows.
\begin{itemize}[leftmargin=*,nosep]
    \item The state space $\mathcal{S} = \{s^0, s^1, \dots, s^T\}$.
    \item The transition function deterministically maps $s^t \mapsto s^{t+1}$ and terminates at input $s^T$.
    \item The reward maps to zero for every state action pair.
\end{itemize}

\textbf{Standard Source Coding} We can express the standard source coding problem as follows.
\begin{itemize}[leftmargin=*,nosep]
    \item The state space $\mathcal{S} = \{s\}$.
    \item The action space $\mathcal{A} = \tilde{\mathcal{A}} \cup \{\emptyset\}$.
    \item The transition function deterministically maps to $s$ to $s$ for all $a \in \tilde{A}$ and terminates the game on $\emptyset$.
    \item The reward maps to $-1$ for $a \in \tilde{\mathcal{A}}$ and maps $\emptyset$ to $0$.
\end{itemize}

\textbf{Length Limited Source Coding} Length limited source coding can be captured in the same way as standard source coding with the modifications that $\mathcal{S} = \{s^0, s^1, \dots, s^T\}$, and the transition function maps $s^t \mapsto s^{t+1}$ and terminates on $s^T$, where $T$ is the length limit.

\textbf{Source Coding with Unequal Symbol Costs} Source coding with unequal symbol costs can be captured in the same way as standard source coding with the modification that $\mathcal{R}(\cdot, a)$ returns the negative symbol cost of $a$, rather than returning $-1$, for $a \in \tilde{\mathcal{A}}$.

\section{Maximum Entropy Reinforcement Learning and Minimum Entropy Coupling} 

To address MCGs, we propose a novel algorithm we call MEME.
MEME (and more broadly, any algorithm geared toward MCGs) is faced with two competing incentives.
On one hand, it needs to maximize expected return $\mathcal{R}(Z)$ generated by the MDP trajectory $Z$.
On the other hand, it needs to maximize the amount of information $\mathcal{I}(M; Z)$ communicated to the receiver about the message, so as to maximize the probability of a correct guess.
MEME handles this trade-off using a two step process for constructing the sender's policy.
In the first step, it computes an MDP policy that balances between high cumulative reward and large cumulative entropy using MaxEnt RL.
In the second step, it opportunistically leverages the existing uncertainty in the MaxEnt RL policy by coupling its probability mass with the posterior over messages in such a way that the expected return does not decrease and the amount of mutual information between the message and trajectory is greedily maximized.
Finally, because the receiver is also privy to the MaxEnt RL policy, it can compute the exact same posterior over the message as the sender and guess the most likely message.
This process is shown in Figure \ref{fig:highlevelmeme} and described in detail below.
Thereafter, we show this procedure possesses desirable guarantees and discuss intuition and scalability.

\begin{figure}
    \centering
    \includegraphics[width=1.0\linewidth]{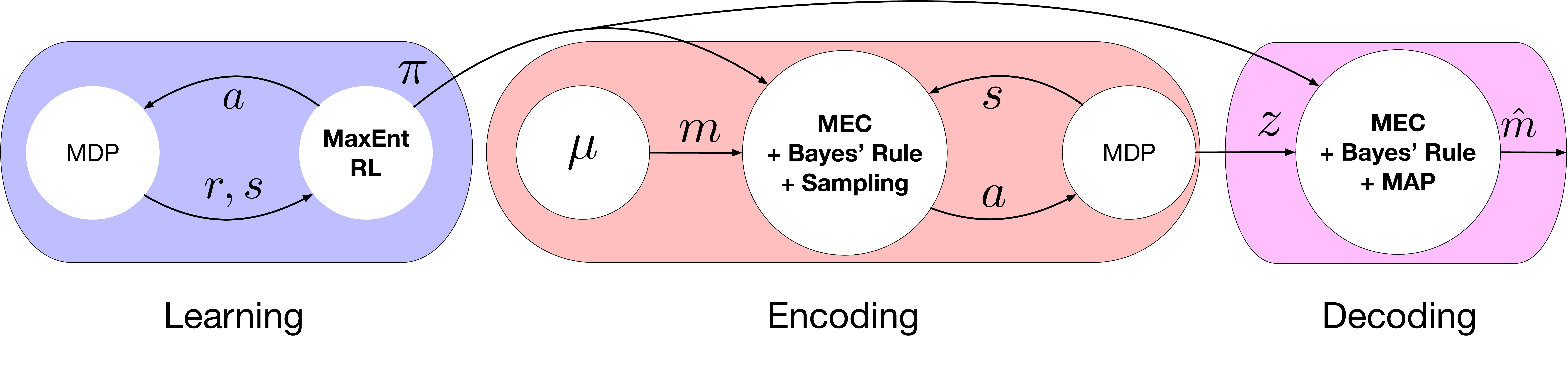}
    \caption{A high level depiction of MEME.}
    \label{fig:highlevelmeme}
\end{figure}

\subsection{Method Description (MEME)}

\paragraph{Step One: Maximum Entropy Reinforcement Learning} In the first step, MEME uses MaxEnt RL to construct an MDP policy $\pi$.
This policy is an MDP policy, not an MCG policy, and therefore does not depend on the message.
Note that this policy depends on the choice of temperature $\alpha$ used for the MaxEnt RL algorithm.

\paragraph{Step Two: Minimum Entropy Coupling} 
In the second step, at execution time, MEME constructs a message-conditional policy online using MECs, as shown in Algorithm \ref{alg:mec-sender}.
Say that, up to time $t$, the sender is in state $s^t$, history $h^t$ and has played according to the state and message conditional policy $\pi_{\mid M}^{:t}$ so far.
Let 
\[b^t = \mathcal{P}(M \mid h^t, \pi_{\mid M}^{:t})\] be the posterior over the message, conditioned on the history and the historical policy.
MEME performs a MEC between the posterior over the message $b^t$ and distribution over actions $\pi(s^t)$, as determined by the MDP policy.
Let $\nu = \text{MEC}(b^t, \pi(s^t))$ denote joint distribution over messages and actions resulting from the coupling.
Then MEME sets the sender to act according to the message conditional distribution 
\[\pi^t_{\mid M}(s^t, m) = \nu(A^t \mid M=m)\] 
of the coupling distribution $\nu = \text{MEC}(b^t, \pi(s^t))$, as is described in Algorithm \ref{alg:mec-sender}.

\begin{algorithm}
   \caption{MEME (Sender)}
   \label{alg:mec-sender}
\begin{algorithmic}
   \STATE {\bfseries Input:} MaxEnt MDP policy $\pi$, MCG $G_{\text{MCG}}$
   \STATE message $m \gets \text{reset}(G_{\text{MCG}})$ // observe message
   \STATE belief $b \gets G_{\text{MCG}}.\mu$
   \STATE state $s \gets G_{\text{MCG}}.s^0$
   \WHILE{sender's turn}
        \STATE joint $\nu \gets \mbox{minimum\_entropy\_coupling}(b, \pi(s))$
        \STATE decision rule $\pi_{\mid M}(s) \gets \nu(A \mid M)$
        \STATE sender action $a \sim \pi_{\mid M}(s, m)$
        \STATE new belief $b \gets \text{posterior\_update}(b, \pi_{\mid M}(s), a)$
        \STATE next state $s \gets G_{\text{MCG}}.\text{step}(a)$
   \ENDWHILE
\end{algorithmic}
\end{algorithm}
Given the sender's MDP trajectory, MEME's receiver uses the sender's MDP policy and MEC procedure to reconstruct the sender's message conditional policy along the trajectory;
thereafter, the receiver computes the posterior and guesses the maximum a posteriori (MAP) message, as shown in Algorithm \ref{alg:mec-receiver}.

\begin{algorithm}
   \caption{MEME (Receiver)}
   \label{alg:mec-receiver}
\begin{algorithmic}
   \STATE {\bfseries Input:} MaxEnt MDP policy $\pi$, MCG $G_{\text{MCG}}$
   \STATE MDP trajectory $z \gets G_{\text{MCG}}.\text{receiver\_observation}()$
   \STATE belief $b \gets G_{\text{MCG}}.\mu$
   \FOR{$s, a \in z$}
    \STATE joint $\nu \gets \mbox{minimum\_entropy\_coupling}(b, \pi(s))$
    \STATE decision rule $\pi_{\mid M}(s) \gets \nu(A \mid M)$
    \STATE new belief $b \gets \text{posterior\_update}(b, \pi_{\mid M}(s), a)$
   \ENDFOR
   \STATE estimated message $\hat{m} \gets \text{arg max}_{m'}b(m')$
   \STATE $G_{\text{MCG}}.\text{step}(\hat{m})$
\end{algorithmic}
\end{algorithm}

\begin{figure*}
    \centering
    \includegraphics[width=1.0\linewidth]{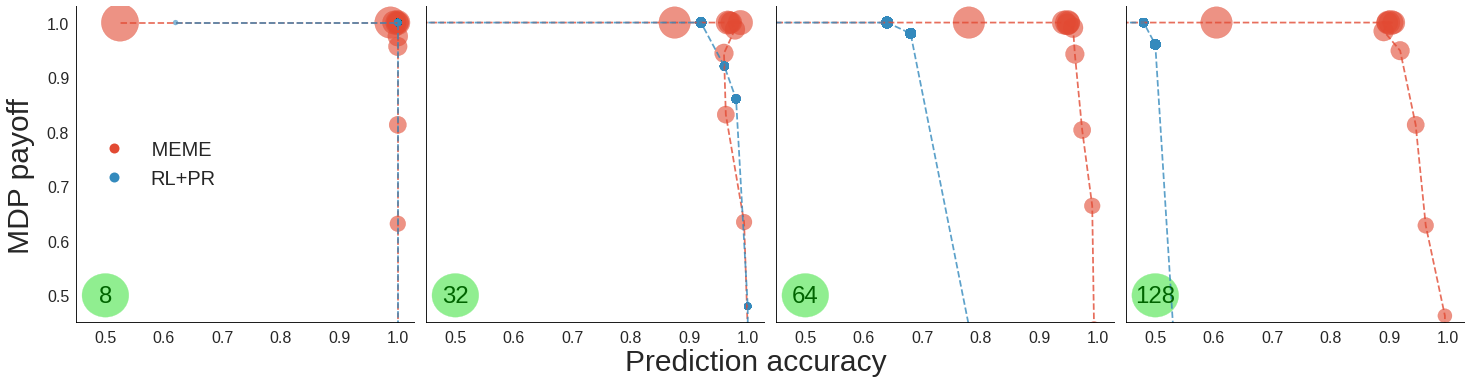}
    \caption{Results for MEME and RL+PR on CodeGrid with varying message space sizes.
    }
    \label{fig:refgrid}
\end{figure*}

\subsection{Method Analysis}

MEME possesses guarantees both concerning the return $\mathcal{R}(Z)$ generated by the MDP and concerning the amount of information communicated $\mathcal{I}(M; Z)$.

\begin{prop} \label{prop:ap}
At any state, in expectation over messages, the message conditional policy $\pi_{\mid M}$ selects actions with the same probabilities as the MDP policy $\pi$.
\end{prop}
\begin{proof}
Fix an arbitrary state $s$.
Let $b$ be a posterior over the message induced by the sender's message conditional policy on the way to $s$.
Then recall that MEME uses the distribution $\text{MEC}(b, \pi(s))$ to select its action.
Because MECs guarantee that the resultant joint distribution marginalizes correctly, it follows directly that MEME must select its actions with the same probabilities as $\pi(s)$.
\end{proof}

Note that Proposition \ref{prop:ap} only holds in expectation over messages.
For particular messages, the distribution over actions will generally differ from the MaxEnt policy.

\begin{prop} \label{prop:er}
In expectation over messages, the expected return generated from the MDP by the message conditional policy $\pi_{\mid M}$ is equal to that of the MaxEnt policy $\pi$. That is,
\[\mathbb{E} \left[\mathcal{R}(Z) \mid  \pi_{\mid M}\right] = \mathbb{E} \left[\mathcal{R}(Z) \mid \pi\right].\]
\end{prop}
\begin{proof}
It follows from Proposition \ref{prop:ap} that all trajectories are generated with the same probabilities and therefore that the expected returns are the same.
\end{proof}

Note that, as with Proposition \ref{prop:ap}, Proposition \ref{prop:er} only holds in expectation over messages.
For particular messages, the expected return of the message conditional $\pi_{\mid M}$ will generally differ from that of the MaxEnt policy $\pi$.

\begin{prop} \label{prop:mi}
At each decision point, MEME greedily maximizes the mutual information $I(M; H^{t+1} \mid b^t, h^t)$ between the message $M$ the history at the next time step $H^{t+1}$, given the contemporaneous posterior and history, subject to Proposition \ref{prop:ap}.
\end{prop}

\begin{proof}
For conciseness, we leave the conditional dependence on $b^t$ and $h^t$ implicit in the argument.
First, we claim that $\mathcal{I}(M; H^{t+1}) = \mathcal{I}(M; A^t)$.
To see this, first consider that $H^{t+1}  \equiv (h^t, A^t, S^{t+1})$.
This means we have $\mathcal{I}(M; H^{t+1}) = \mathcal{I}(M; (A^t, S^{t+1}))$ since we are conditioning on $h^t$.
Now, consider that because the message influences the next state only by means of the action, we have the causal graph $M \to A^t \to (A^t, S^{t+1})$, which implies that $M \perp (A^t, S^{t+1}) \mid A^t$.
Also, we trivially have $M \perp A^t \mid (A^t, S^{t+1})$.
These conditional independence properties allow us to apply the data processing inequality:
$X \perp Z \mid Y \Rightarrow \mathcal{I}(X; Y) \geq \mathcal{I}(X; Z)$.
Applying it in both directions yields $\mathcal{I}(M; (A^t, S^{t+1})) = \mathcal{I}(M; A^t)$.

Now consider that we can rewrite mutual information using the equality 
\begin{align*}
    \mathcal{I}(M; A^t) = \mathcal{H}(M) + \mathcal{H}(A^t) - \mathcal{H}(M, A^t).\label{eq:decomp}
\end{align*}

The first term $\mathcal{H}(M)$ is exogenous by virtue of being determined by $b^t$.
The second term is exogenous by virtue of being subject to Proposition \ref{prop:ap}.
The third term is the joint entropy between $M$ and $A^t$, which is exactly what a minimum entropy coupling minimizes.
\end{proof}

\paragraph{MaxEnt RL} The analysis thus far concerns justifying MEME's iterative minimum entropy coupling procedure with formal guarantees.
While we do not provide guarantees with analogous rigor for the MaxEnt RL component, there is nevertheless strong justification for its usage.
Recall that MEME is concerned with i) maximizing the MDP return and ii) greedily maximizing the amount of information communicated at each decision point.
These two quantities are 
$\textcolor{blue}{\mathcal{R}(Z)}$ and $\textcolor{blue}{\mathcal{H}(A^t)} + \mathcal{H}(M) \textcolor{red}{-\mathcal{H}(M, A^t)}$
where the second quantity corresponds to the one-step mutual information, as was shown in the proof of Proposition \ref{prop:mi}.
Consider the four terms here.
The term in black text, $\mathcal{H}(M)$, is exogenously specified, and therefore can be ignored for optimization purposes.
The term in red text, $\textcolor{red}{-\mathcal{H}(M, A^t)}$, is exactly the quantity being maximized by minimum entropy coupling.
Elegantly, the remaining terms (shown in blue text), are exactly the quantities that MaxEnt RL is concerned with maximizing---the MDP return $\textcolor{blue}{\mathcal{R}(Z)}$ and the entropy over actions $\textcolor{blue}{\mathcal{H}(A^t)}$.

\begin{figure*}
    \centering
    \includegraphics[width=1.0\linewidth]{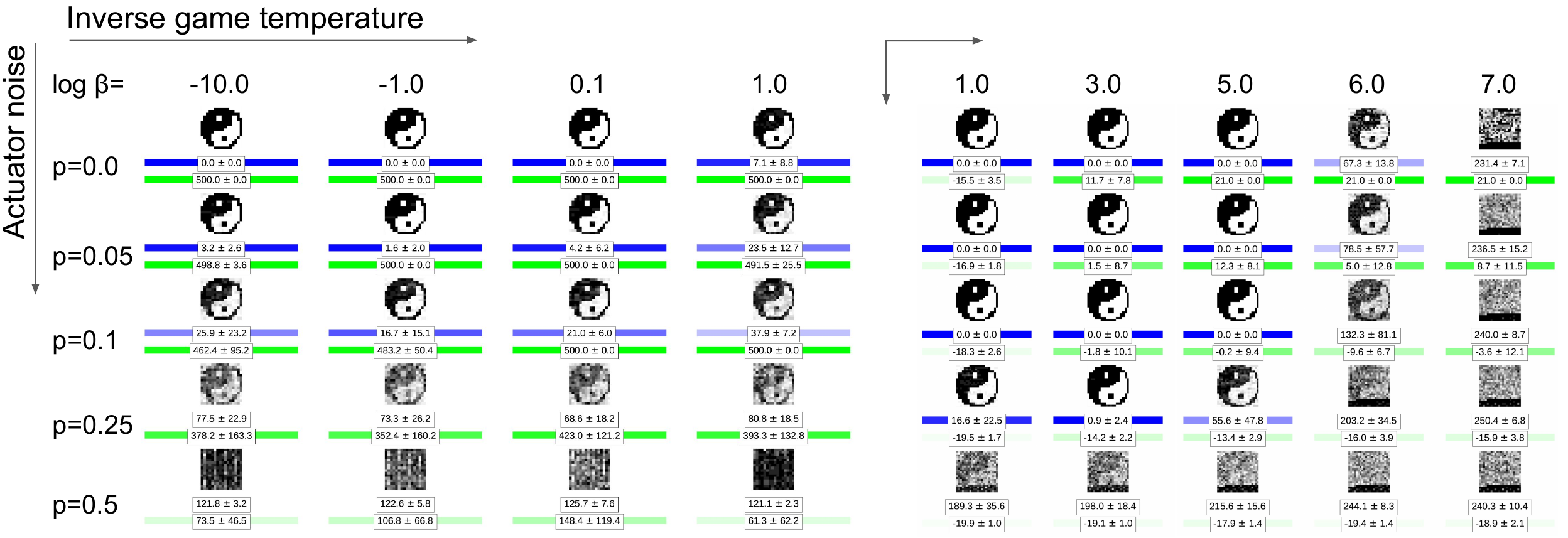}
    \caption{Results for MEME on CodeCart (left) and CodePong (right) with varying amounts of actuator noise and temperatures. The blue bars indicate pixel errors, and the green bars show scores.
    }
    \label{fig:cartpole_and_pong}
\end{figure*}

MaxEnt RL is also advantageous in that it allows MEME to trade off between MDP return and communication using the temperature $\alpha$.
For large temperatures, MEME places more emphasis on maximizing $\mathcal{H}(A^t)$ (and thereby $\mathcal{I}(M; Z)$), while for smaller temperatures, MEME places more emphasis on maximizing $\mathcal{R}(Z)$.
The appropriate choice of $\alpha$ for a particular setting will depend on the message priority $\zeta$.
Code for MEME illustrating the functionality of $\alpha$ applied to the example from Figure \ref{fig:mcgex} can be found at \url{https://bit.ly/36I3LDm} (the code is executable in browser).

\subsection{Scalability} Historically, performing even approximate versions of the minimum entropy couplings described in Algorithms \ref{alg:mec-sender} and \ref{alg:mec-receiver} would have been impossible due to a lack of polynomial minimum entropy coupling methods.
However, with the recent advent of efficient approximation algorithms \cite{mec-appx}, it is now possible to perform near-optimal couplings in $O(N \log N)$ time.
Still, $O(N \log N)$ is not scalable to very large message spaces.
To accommodate this fact, for large message spaces, we use a factored representation $\mathbb{M} \subset \mathbb{M}_1 \times \dots \times \mathbb{M}_k$ and a corresponding factored belief $(b_1, \dots, b_k)$, where each $b_j$ tracks the posterior over $\mathbb{M}_j$.
Specifically, at each time step, we perform a MEC between $\pi(s^t)$ and the block $b_j = \text{arg max}_{b_j} \mathcal{H}(b_j)$ having the largest entropy.\footnote{Pseudocode for this procedure is provided in the appendix.}
By doing so, the size of the space over which the minimum entropy coupling is performed is greatly reduced, allowing us to scale to extremely large message spaces while retaining the guarantees of Proposition \ref{prop:er}.
While Proposition \ref{prop:mi} is no longer satisfied---because i) we use approximate minimum entropy coupling, ii) we only couple a single message block at a time, and iii) the posterior is only approximate if the message prior is not factorable---mutual information remains high in practice, as we will see in the experiments.

\section{Experiments}

We investigate the efficacy of MEME on MCGs based on a gridworld, Cartpole and Pong.
For all three we use MaxEnt Q-learning for our MDP policy.
Our codebase is available at \url{https://github.com/schroederdewitt/meme}.

\paragraph{CodeGrid} In our gridworld MCG, which we call CodeGrid, the sender is placed on a $4 \times 4$ grid in which it can move left, right, up and down. 
The sender starts at position $(1,1)$ and must to move to $(4,4)$ within $8$ time steps to achieve the maximal MDP return of $1$. Otherwise, the game terminates after $8$ time steps and the sender receives a payoff of $0$.

For our baseline, we trained an RL agent to play as the MCG sender.
We paired this RL agent with a perfect receiver, meaning that, for each episode, the receiver computed the exact posterior over the message based on the sender's current policy and guessed the MAP, both during training and testing.
We abbreviate this baseline as RL+PR (where PR stands for perfect receiver).
RL+PR is inspired by related work such as \citet{strouse,diayn}.
Pseudocode for the RL+PR baseline can be found in the appendix.

We show results for MEME and RL+PR across a variety of settings in Figure \ref{fig:refgrid}.
The column of the figure corresponds to the cardinality of the message space; the exact size is listed in the green bubble.
We use a uniform marginal over messages in each case.
The $x$-axis corresponds to the proportion of the time that the receiver correctly guesses the message.
The $y$-axis corresponds to the MDP payoff (in this case whether the sender reached the opposing corner of the grid within the time limit).
Both MEME and RL+PR possess mechanisms to trade-off between these two goals.
MEME can adjust its temperature $\alpha$, while RL+PR can adjust the value of $\zeta$ used during training.
Figure \ref{fig:refgrid} show the Pareto curve for each.
For MEME, the size of the circle corresponds to the inverse temperature $\beta = 1/\alpha$.

For the settings with 8 messages and 32 messages, we observe that both MEME and RL+PR achieve strong performance---achieving optimal or nearly optimal MDP returns and optimal or nearly optimal transmission.
However, as the size of the message space increases, the performance of RL+PR falls off sharply.
Indeed, for the 128 message setting, RL+PR is only able to correctly transmit the message slightly more than half the time.
RL+PR's inability to achieve strong performance in these cases may be a result of the fact that communication protocols violate the locality assumption of approximate dynamic programming approaches.
In contrast, we observe that MEME, which constructs its protocol using MEC, remains optimal or near optimal both the 64 and the 128 message settings.

\paragraph{CodeCart and CodePong} While the CodeGrid experiments show that MEME can outperform an obvious approach to MCGs, it remains to be determined whether MEME can perform well at scale.
Toward the end of making this determination, we consider MCGs based on Cartpole and Pong.
In these MCGs, the message spaces are the sets of $16 \times 16$ and $22 \times 22$ binary images, respectively, each with a uniform prior.
The cardinality of these spaces ($>10^{77}$ and $>10^{145}$) is exceedingly large.
In fact, it is not immediately clear how an RL+PR-like approach could even be scaled to this setting.
On the other hand, MEME is easily adaptable to this setting, using the factorization scheme suggested in the method section.

We also include results with variable amounts of \textit{actuator noise}, i.e., with some probability $p$, the environment executes a random action, rather than the one it intended.
Actuator noise models the probability of error during transmission and arises naturally in many settings involving communication, such as noisy channel coding.
In this setting, the receiver may only observe the action executed by the environment, rather than the one intended by the sender. While this setting is not technically an MCG and Proposition \ref{prop:mi} no longer holds (even with perfect MEC), we show MEME to be a strong heuristic in these settings.

We show the results in Figure \ref{fig:cartpole_and_pong} for CodeCart and CodePong, respectively.
For both plots, the $y$-axis corresponds to the amount of actuator noise (lower is more noise). The $x$-axis corresponds the inverse temperature value $\beta=1/\alpha$ (further right is colder temperature, meaning there is more emphasis on MDP expected return).
Each entry contains a decoded yin-yang symbol with the corresponding temperature and actuator noise.
Each entry also lists the $\ell_1$ pixel error (blue) and the MDP expected return (green), along with corresponding standard errors over 10 runs; a brighter color corresponds to better performance.

Remarkably, for both CodeCart and CodePong, we observe that, when there is no actuator noise, MEME is able to perfectly transmit images while achieving the maximum expected return in the MDP.
For CodeCart, this occurs at $\log \beta \in \{-10, -1, 0.1\}$; for CodePong, it occurs at $\log \beta \in \{4, 5\}$.
Interestingly, as the amount of actuator noise increases to a non-zero value, the effect on performance differs between CodeCart and CodePong.
In CodeCart, MEME continues to achieve maximal performance in the MDP, but begins to accumulate errors in the transmission.
In contrast, in CodePong, MEME continues to transmit the message with perfect fidelity, but begins to lose expected return from the MDP.
This suggests that accidentally taking a random action is costlier in Pong than it is in Cartpole, but that, in an informal sense, Pong has more bandwidth to transmit information.
That said, in both cases, the decay in performance is graceful---for example, with $p=0.05$, the decrease in the visual quality of the transmission for Cartpole ($\log \beta = -1$) is difficult to even perceive, while the CodePong sender still manages to win games roughly 80\% of the time ($\log \beta = 5$).
The performance in both settings continues to deteriorate up to $p=1/2$, at which point MEME is neither able to perform adequately on the MDP nor to transmit a clear symbol.

Videos of play are included in the supplemental material.
\section{Conclusions and Future Work}

This work introduces a new problem setting called Markov coding games, which generalize referential games and source coding and are related to channel coding problems and decentralized control.
We contribute an algorithm for MCGs called MEME and show that MEME possesses provable guarantees regarding both the return generated from the MDP and the amount of information content communicated.
Experimentally, we show that MEME significantly outperforms an RL baseline with a perfect receiver.
Finally, we show that MEME is able to scale to extremely large message spaces and transmit these messages with high fidelity, even with some actuator noise, suggesting that it could be robust in real world settings. 

\paragraph{Security} One direction for future work is applying MEME as a method for secure communication. 
Consider a situation in which an agent would like to communicate privileged information to an allied agent, while under adversarial observation. 
One possibility is for the agent to communicate using an encrypted channel. 
However, doing so informs the adversary that the agent possibly has access to, and is communicating, sensitive information. 
This is undesirable in many contexts, as it may result in the adversary launching attacks in order to determine the sensitive information or prevent it from being communicated. 
An alternative is for the agent to perform a task as it normally would and allow the ally (the receiver) to interpret the message from the trajectory. 
This allows the agent to deny that communication ever took place, due to Proposition \ref{prop:ap}.

\paragraph{Iterative Minimum Entropy Coupling} As a subprocedure of MEME, this work made contributions in showing how distributions can be iteratively approximately minimum entropy coupled.
Specifically, Step 2 of factored MEME demonstrates how any distribution factorable into blocks with small supports can be iteratively coupled with an autoregressively specified distribution.
Another direction for future work is to extend this approach so that it may be applied more generally, with two autoregressively specified distributions (rather than requiring one be factorable). 
Such an approach could improve MEME's performance in settings in which the belief prior is not factorable.
It could prove useful in more general contexts as a tool for producing low entropy couplings of distributions with large supports.

\looseness=-1
\section{Acknowledgements}
We thank Frans A. Oliehoek, Alexander Robey, and Yiding Jiang for helpful discussions. This work was supported by the Bosch Center for Artificial Intelligence; the Cooperative AI Foundation; the Cyber Defence Campus, Science and Technology, Armasuisse, Switzerland; the Microsoft Research PhD Scholarship Program; the Microsoft Research EMEA PhD Award; the European Research Council under the European Union’s Horizon 2020 research and innovation programme (grant agreement \#637713); and NVIDIA.

\bibliography{example_paper}
\bibliographystyle{icml2022}

\newpage
\onecolumn
\raggedbottom
\appendix

\section{Experimental Details}

For CodeGrid, we use a policy parameterized by neural network with two fully-connected layers of hidden dimension $64$, each followed by a ReLu activation \citep{nair_rectified_2010}. 
For CodePong and CodeCart, we use a convolutional encoder with three layers of convolutions (number of channels, kernel size, stride) as follows: ($32$,$8$,$4$), ($64$,$4$,$2$), ($64$,$3$,$1$). This is followed by a fully connected layer \citep{mnih_human-level_2015}. 
We use ReLu activations after each layer, note that we do not use any max-pooling. 
For CodeGrid and CodePong, layer weights are randomly initialized using PyTorch $1.7$ \citep{paszke_automatic_2017} defaults.
We used 200k training episodes CodeGrid and 2M training episodes for CodePong.
For CodeCart, we initialized weights according to an optimally-trained DQN policy included in \textit{rl-baselines3-zoo}\footnote{\url{https://github.com/DLR-RM/rl-baselines3-zoo}} and individually fine-tune for each $\beta$ for $15k$ steps.
For all environments, we used the Adam optimizer with learning rate $10^{-4}$, $\beta_1=0.9,\ \beta_2=0.999,\ \epsilon=10^{-8}$ and no weight decay.
The results in Figures \ref{fig:refgrid} and \ref{fig:cartpole_and_pong} are averages over 10 rollouts; for Figure \ref{fig:cartpole_and_pong}, uncertainty is shown with standard error.

\textbf{CodeGrid}

\begin{figure}[h!]
\centering
  \includegraphics[width=0.35\columnwidth]{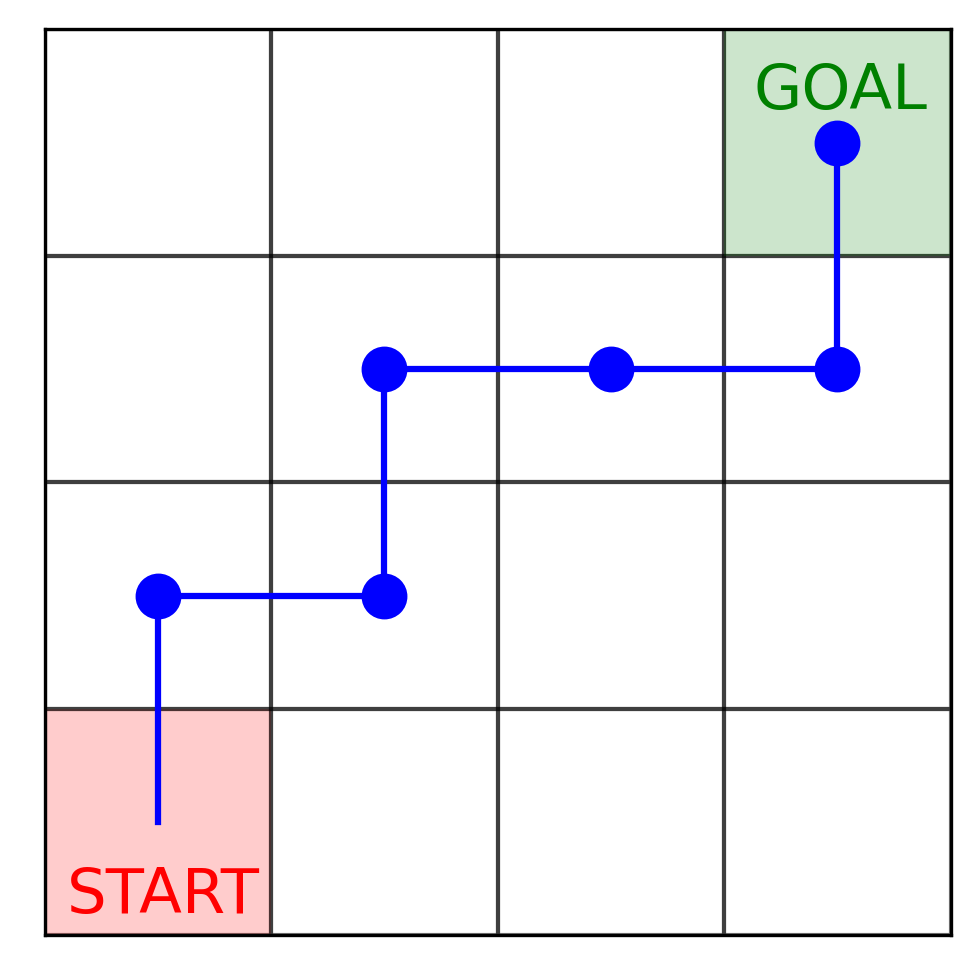}
  \hspace{0.5cm}
  \includegraphics[width=0.35\columnwidth]{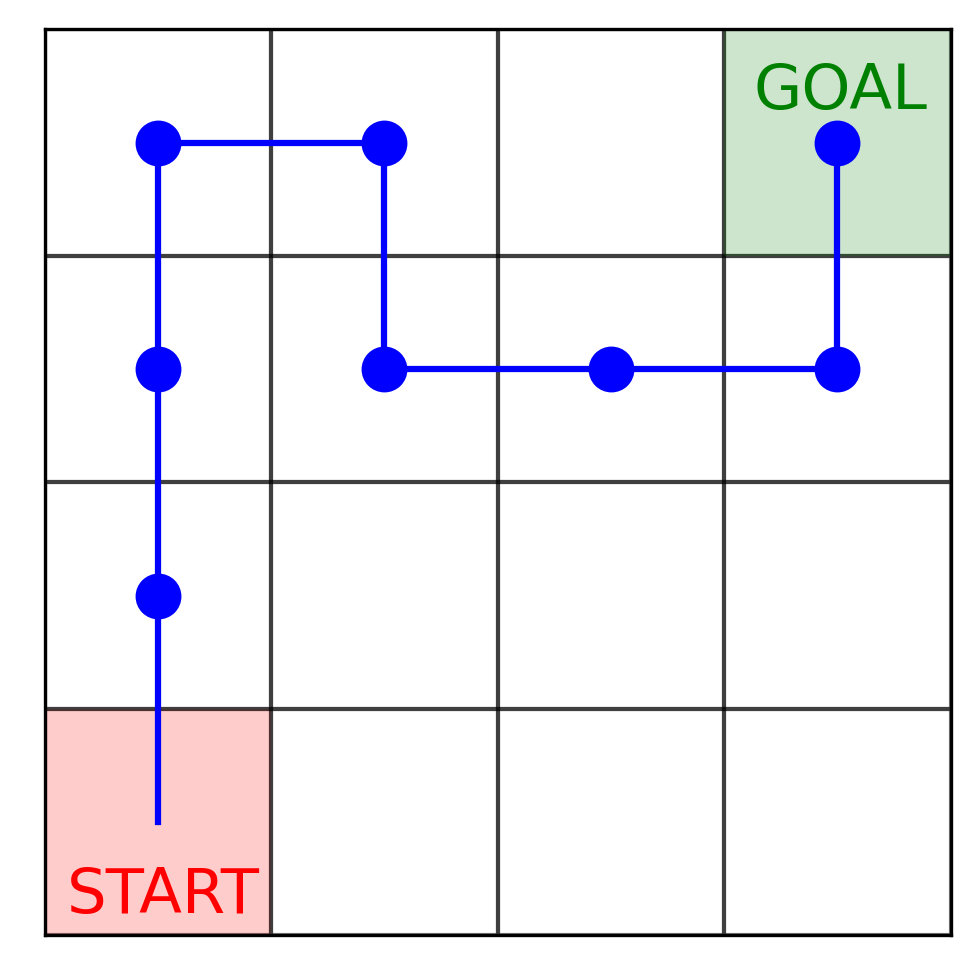}
  \caption{An illustration of two possible CodeGrid trajectories, one of length $6$ (left) and one of length $8$ (right). 
  }
\label{fig:refgrid2}
\end{figure}

\textbf{CodeCart}
For CodeCart we use the standard action space (move left and move right).
There are no redundant actions that the agent can use to communicate.

\textbf{CodePong}
For CodePong we use a reduced action space of size two (move up and move down).
There are no redundant actions that the agent can use to communicate.

\section{Algorithmic Details}

\subsection{Factored MEME}
\begin{algorithm}[H]
   \caption{Factored MEME}
   \label{alg:examplefactored}
\begin{algorithmic}
   \STATE // Step 1: Compute MDP policy
   \STATE {\bfseries Input:} MDP $G_{\text{MDP}}$, temperature $\alpha$
   \STATE MaxEnt MDP policy $\pi \gets \mbox{MaxEntRL}(G_{\text{MDP}}, \alpha)$
   \STATE
   \STATE // Step 2: Play Sender's Part of MCG episode
   \STATE {\bfseries Input:} MDP policy $\pi$, MCG $G_{\text{MCG}}$
   \STATE message $m \gets \text{reset}(G_{\text{MCG}})$ // observe message from environment
   \STATE belief $b \gets G_{\text{MCG}}.\mu$
   \STATE state $s \gets G_{\text{MCG}}.s^0$
   \WHILE{sender's turn}
        \STATE active block index $i \gets \mathop{\arg\max_j}\left\{\mathcal{H}(b_j) | b_j\in b\right\}$
        \STATE joint $\nu \gets \mbox{minimum\_entropy\_coupling}(b_i, \pi(s))$
        \STATE decision rule $\pi_{\mid M} \gets \nu(A \mid M)$
        \STATE sender action $a \sim \pi_{\mid M}(m_i)$
        \STATE new belief $b_i \gets \text{posterior\_update}(b_i, \pi_{\mid M}, a)$
        \STATE next state $s \gets G_{\text{MCG}}.\text{step}(a)$
   \ENDWHILE
   \STATE
   \STATE // Step 3: Play Receiver's Part of MCG episode
   \STATE {\bfseries Input:} MDP policy $\pi$, MCG $G_{\text{MCG}}$
   \STATE MDP trajectory $z \gets G_{\text{MCG}}.\text{receiver\_observation}()$
   \STATE belief $b \gets G_{\text{MCG}}.\mu$
   \FOR{$s, a \in z$}
    \STATE active block index $i \gets \mathop{\arg\max_j}\left\{\mathcal{H}(b_j) | b_j\in b\right\}$
    \STATE joint $\nu \gets \mbox{minimum\_entropy\_coupling}(b_i, \pi(s))$
    \STATE decision rule $\pi_{\mid M} \gets \nu(A \mid M)$
    \STATE new belief $b \gets \text{posterior\_update}(b, \pi_{\mid M}, a)$
   \ENDFOR
   \FOR{$i$}
   \STATE estimated $i$th message block $\hat{m}_i \gets \text{arg max}_{m'_i}b_i(m'_i)$
   \ENDFOR
   \STATE $G_{\text{MCG}}.\text{step}(\hat{m})$
\end{algorithmic}
\end{algorithm}

\subsection{RL+PR baseline}

\label{app:iql_pbs}
\begin{algorithm}[H]
  \caption{RL+PR baseline}
  \label{alg:iql_pbs}
\begin{algorithmic}
  \STATE // Play Episode
  \STATE {\bfseries Input:} MCG $G_{\text{MCG}}$
  \STATE state $s \gets s^0$
  \STATE message $m \gets \text{sample}(\mu)$
  \STATE belief $b \gets \mu$
  \WHILE {True}
  \STATE action $a \gets \text{sample}(\pi(s, m))$
  \STATE belief $b \gets \text{posterior\_update}(b, \pi, a)$
  \STATE new state $s' \gets \text{sample}(\mathcal{T}(s, a))$
  \IF{$s'$ non-terminal}
  \STATE $\text{add\_to\_buffer}(s, a, \mathcal{R}(s, a), s')$
  \STATE state $s \gets s'$
  \ELSE
  \STATE break
  \ENDIF
  \ENDWHILE
  \STATE $\text{add\_to\_buffer}(s, a, \mathcal{R}(s, a) + \zeta \max_{m' \in \mathcal{M}} b(m'), \emptyset)$
\end{algorithmic}
\end{algorithm}

\subsection{Min Entropy Joint Distribution Algorithm outputting a sparse representation of $M$ \cite{mec-appx}}
\label{sec:minent}

\begin{algorithm}[H]
   \caption{Min Entropy Joint Distribution}
   \label{alg:min_ent}
\begin{algorithmic}
   \REQUIRE prob. distributions $\mathbf{p}=\left(p_1,\dots,p_n\right)$ and $\mathbf{q}=\left(q_1,\dots,q_n\right)$
   \ENSURE A Coupling $\mathbf{M}=[m_{ij}]$ of $\mathbf{p}$ and $\mathbf{q}$ in sparse representation $\mathbf{L}=\left\{\left(m_{ij},(i,j)\right)|m_{ij}\neq 0\right\}$
   \STATE $\quad$
   \STATE \textbf{if} $\mathbf{p}\neq\mathbf{q}$, let $i=\max\left\{j|p_j\neq q_j\right\}$;  \textbf{if} $p_i<q_i$ \textbf{then swap } $\mathbf{p}\leftrightarrow\mathbf{q}$.
   \STATE $\mathbf{z}=(z_1,\dots,z_n)\leftarrow\mathbf{p}\land\mathbf{q}$, $\mathbf{L}\leftarrow\emptyset$
   \STATE CreatePriorityQueue$\left(\mathcal{Q}^{(row)}\right),\ \text{qrowsum}\leftarrow 0$
   \STATE CreatePriorityQueue$\left(\mathcal{Q}^{(col)}\right),\ \text{qcolsum}\leftarrow 0$
   \FOR{$i=n$ \textbf{downto} $1$}
    \STATE $z_i^{(d)}\leftarrow z_i,\ z_i^{(r)}\leftarrow 0$
    \IF{$qcolsum+z_i>q_i$}
        \STATE $\left(z_i^{(d)},z_i^{(r)},I,qcolsum\right)\leftarrow \text{Lemma3-Sparse}\left(z_i,q_i,\mathcal{Q}^{(col)},qcolsum\right)$
    \ELSE
        \WHILE{$\mathcal{Q}^{(col)}\neq\emptyset$}
            \STATE $(m,l)\leftarrow \text{ExtractMin}(\mathcal{Q}^{(col)})$,
            \STATE $qcolsum\leftarrow qcolsum-m$,
            \STATE $\mathbf{L}\leftarrow\mathbf{L}\cup\left\{(m,(l,i))\right\}$
        \ENDWHILE
    \ENDIF
    \IF{$qrowsum+z_i>p_i$}
        \STATE $(z_i^{(d)},\ z_i^{(r)},I,qrowsum)\leftarrow \text{Lemma3-Sparse}(z_i,p_i,\mathcal{Q}^{(row)},qrowsum)$
        \STATE \textbf{for each} $(m,l)\in I$ \textbf{do} $\textbf{L}\rightarrow\textbf{L}\cup\left\{(m,(i,l))\right\}$
        \STATE \textbf{if} $z_i^{(r)}>0$ \textbf{then} $\text{Insert}\left(\mathcal{Q}^{(row)},(z_i^{(r)},i)\right)$
        \STATE $qrowsum\leftarrow qrowsum + z_i^{(r)}$
    \ELSE
        \WHILE{$\mathcal{Q}^{(row)}\neq\emptyset$}
            \STATE $(m,l)\leftarrow \text{ExtractMin}(\mathcal{Q}^{(row)})$
            \STATE $qrowsum\leftarrow qrowsum-m$
            \STATE $\mathbf{L}\leftarrow\mathbf{L}\cup\left\{(m,(i,l))\right\}$
        \ENDWHILE
    \ENDIF
    \STATE $\mathbf{L}\leftarrow\mathbf{L}\cup\left\{(z_i^{(d)},(i,i))\right\}$
  \ENDFOR
\end{algorithmic}
\end{algorithm}

\begin{algorithm}
   \caption{Lemma$3$-Sparse}
   \label{alg:lemma3_sparse}
\begin{algorithmic}
   \REQUIRE real $z>0,\ x\geq0$, and priority queue $\mathcal{Q}$ s.t.$\left(\sum_{(m,l)\in\mathcal{Q}}m\right)=qsum$ and $qsum=x\geq z$ 
   \ENSURE $z^{(d)},z^{(r)}\geq 0$, and $I\subseteq\mathcal{Q}$ s.t. $z^{(d)}+z^{(r)}=z$, and $z^{(d)}+\sum_{(m,l)\in I}m=x$
   \STATE $I\leftarrow\emptyset,\ sum\leftarrow 0$
   \WHILE{$\mathcal{Q}\neq\emptyset$ \textbf{and} $sum+Min(\mathcal{Q})<x$}
    \STATE $(m,l)\leftarrow\text{ExtractMin}(\mathcal{Q}),\ qsum\leftarrow qsum - m$
    \STATE $I\leftarrow I\cup\left\{(m,l)\right\},\ z^{(r)}\leftarrow z-z^{(d)}$
   \ENDWHILE
   \STATE $z^{(d)}\leftarrow x -sum,\ z^{(r)}\leftarrow z-z^{(d)}$
   \STATE \textbf{return} $\left(z^{(d)},z^{(r)},I, qsum\right)$
\end{algorithmic}
\end{algorithm}

\end{document}